\documentclass{article}
\usepackage{ijcai20}

\usepackage{times}
\usepackage{soul}
\usepackage{url}
\usepackage[hidelinks]{hyperref}
\usepackage[utf8]{inputenc}
\usepackage[small]{caption}
\usepackage{graphicx}
\usepackage{amsmath,bm}
\usepackage{amsthm}
\usepackage{amsfonts}
\usepackage{booktabs}
\usepackage{algorithm}
\usepackage{algorithmic}
\usepackage{subfigure}
\usepackage{color}
\usepackage{multirow}
\usepackage{balance}
\usepackage[english]{babel}

\newtheorem{problem}{Problem}
\newtheorem{theorem}{Theorem}[section]

\newtheorem{lemma}[theorem]{Lemma}

\title{Feature Interaction-aware Graph Neural Networks}
\author{
Kaize Ding$^1$\and
Yichuan Li$^1$\and
Jundong Li$^2$\And
Chenghao Liu$^3$\And
Huan Liu$^1$\\
\affiliations
$^1$Computer Science and Engineering, Arizona State University, USA\\
$^2$Electrical and Computer Engineering, University of Virginia, USA\\
$^3$Information Systems, Singapore Management University, Singapore\\
\emails
\{kaize.ding, yichuan1\}@asu.edu,
jl6qk@virginia.edu,
chliu@smu.edu.sg,
huan.liu@asu.edu
}

\begin{document}
\maketitle

\begin{abstract}
Inspired by the immense success of deep learning, graph neural networks (GNNs) are widely used to learn powerful node representations and have demonstrated promising performance on different graph learning tasks. However, most real-world graphs often come with high-dimensional and sparse node features, rendering the learned node representations from existing GNN architectures less expressive. In this paper, we propose \textit{Feature Interaction-aware Graph Neural Networks (FI-GNNs)}, a plug-and-play GNN framework for learning node representations encoded with informative feature interactions. Specifically, the proposed framework is able to highlight informative feature interactions in a personalized manner and further learn highly expressive node representations on feature-sparse graphs.  Extensive experiments on various datasets demonstrate the superior capability of FI-GNNs for graph learning tasks.






\end{abstract}

\section{Introduction}

Graph-structured data, ranging from social networks to financial transaction networks, from citation networks to gene regulatory networks, have been widely used for modeling a myriad of real-world systems~\cite{hamilton2017inductive}. As the essential key for conducting various network analytical tasks (e.g., node classification, link prediction, and community detection), learning low-dimensional node representations has drawn much research attention. Recently, graph neural networks (GNNs) have demonstrated their remarkable performance in the field of network representation learning~\cite{kipf2016semi,velivckovic2017graph,hamilton2017inductive}. The main intuition behind 
GNNs
is that the node's latent representation  could be integrated by transforming, propagating, and aggregating node features from its local neighborhood.




Despite their enormous success, one fundamental limitation of existing GNNs is that the neighborhood aggregation scheme directly makes use of the raw features of nodes for the message-passing. However, the node features of many real-world graphs could be high-dimensional and sparse~\cite{he2017neural}. In a social network that represents friendship of users, to characterize the profile of a user, his/her categorical predictor variables (e.g., group and tag) could be converted to a set of binary features via one-hot encoding~\cite{tang2008arnetminer,velivckovic2017graph}; In a citation network that represents citation relations between papers, the bag-of-words or TF-IDF models~\cite{zhang2010understanding} are often used to encode papers for obtaining their node attributes. Due to the fact that existing GNN models are not tailored for learning from such feature-sparse graphs, their performance could be largely limited due to the curse of dimensionality~\cite{li2018feature}.

To handle the aforementioned feature-sparse issue, one prevalent way is to leverage the combinatorial interactions among features: for example, we can create a second-order cross-feature set \textit{gender\_income = \{M\_high, F\_high, M\_low, F\_low\}} based on variables \textit{gender = \{M, F\}} and \textit{income = \{high, low\}}. As such, feature interactions could provide more predictive power, especially when features are highly sparse. However, finding such meaningful feature interactions requires intensive engineering efforts and domain knowledge, thus it is almost infeasible to hand-craft all the informative combinations. To automatically learn the feature interactions, various approaches have been proposed in the research community, e.g., Factorization Machines~\cite{rendle2010factorization} combine polynomial regression models with factorization techniques, and DeepCross Networks~\cite{wang2017deep} model the high-order feature interactions based on the multi-layered residual network. Though they have demonstrated their effectiveness in different learning tasks, directly applying them to graph-structured data will inevitably yield unsatisfactory results. The main reason is that those methods focus on attribute-value data which is often assumed to be independent and identically distributed (\textit{i.i.d.}), while they are unable to capture the node dependencies on graphs. Therefore, capturing informative feature interactions is vital for enhancing graph neural networks, but largely remains to be explored.

Towards advancing graph neural networks to learn more expressive node representations, we propose a novel framework: feature interaction-aware graph neural networks (FI-GNNs) in this study. Our framework consists of three modules that seamlessly work together for learning on feature-sparse graphs. Specifically, the \textit{message aggregator} recursively aggregates and compresses raw features of nodes from local neighborhoods. Concurrently, the \textit{feature factorizer} extracts a set of embeddings that encodes factorized feature interactions. 
According to social identity theory~\cite{tajfel2010social}, individuals in a network are highly idiosyncratic and often exhibit different patterns, thus it is critical to consider the individuality and
personality during the feature interaction learning~\cite{chen2019bayesian}. To this end, a personalized-attention module is developed to balance the impacts of different feature interactions on demand of the prediction task. For each node, this module takes the aggregated raw features as a query and employ personalized attention mechanism to accurately highlight informative feature interactions for it. In this way, we are able to learn feature interaction-aware node representations on such feature-sparse graphs. It is worth mentioning that FI-GNN is a plug-and-play framework, which is compatible with arbitrary GNNs: by plugging different GNN models,
our framework can be extended to support learning on various types of graphs (e.g., signed graphs and heterogeneous graphs). To summarize, the major contributions of our work are as follows:


\begin{itemize}
    \item To our best knowledge, we are the first to study the novel problem of feature interaction-aware node representation learning, which addresses the limitation of existing GNN models in handling feature-sparse graphs.

    \item We present FI-GNNs, a novel plug-and-play framework that seamlessly learns node representations with informative feature interactions. 
    
    \item We conduct comprehensive experiments on real-world networks from different domains to demonstrate the effectiveness of the proposed framework.
\end{itemize}

\section{Problem Definition}

To legibly describe the studied problem, we follow the commonly used notations throughout the paper. Specifically, we use lowercase letters to denote scalars 
(e.g., $x$), boldface lowercase letters to denote vectors (e.g., $\mathbf{x}$), boldface uppercase letters to denote matrices (e.g., $\mathbf{X}$), and calligraphic fonts to denote sets (e.g., $\mathcal{V}$).

We denote a graph as $G = (\mathcal{V}, \mathcal{E}, \mathbf{X})$, where $\mathcal V$ is the set of $n$ nodes and $\mathcal{E}$ is the set of $m$ edges. $\mathbf{X} = [\mathbf{x}_1, \mathbf{x}_2, \dots, \mathbf{x}_n] \in \mathbb{R}^{n \times d}$ denotes the features of these $n$ nodes. The $j^{th}$ feature value of node $i$ is denoted as $x_{ij}$. Commonly, the topological structure of a graph can be represented by the adjacency matrix $\mathbf{A} = \{0, 1\}^{n \times n}$, where $a_{ij}=1$ indicates that there is an edge between node $i$ and node $j$; otherwise, $a_{ij}=0$. Here we focus on undirected graphs, though it is straightforward to extend our approach to directed graphs. For other notations, we introduce them in later sections. Formally, our studied problem can be defined as:

\begin{problem}
\textbf{Feature Interaction-aware Node Representation Learning}: Given an input graph $G = (\mathbf{A}, \mathbf{X})$, the model objective is to map nodes $\mathcal{V}$ to latent representations $\mathbf{Z} = [\mathbf{z}_1, ..., \mathbf{z}_n] \in \mathbb{R}^{n \times d'}$, where each node representation $\mathbf{z}_i$ incorporates interactions between its features $\mathbf{x}_i$.
\end{problem}

\section{Proposed Approach}


 \begin{figure*}[t]
    \graphicspath{{figures/}}
    \centering
    \includegraphics[width=0.9\textwidth]{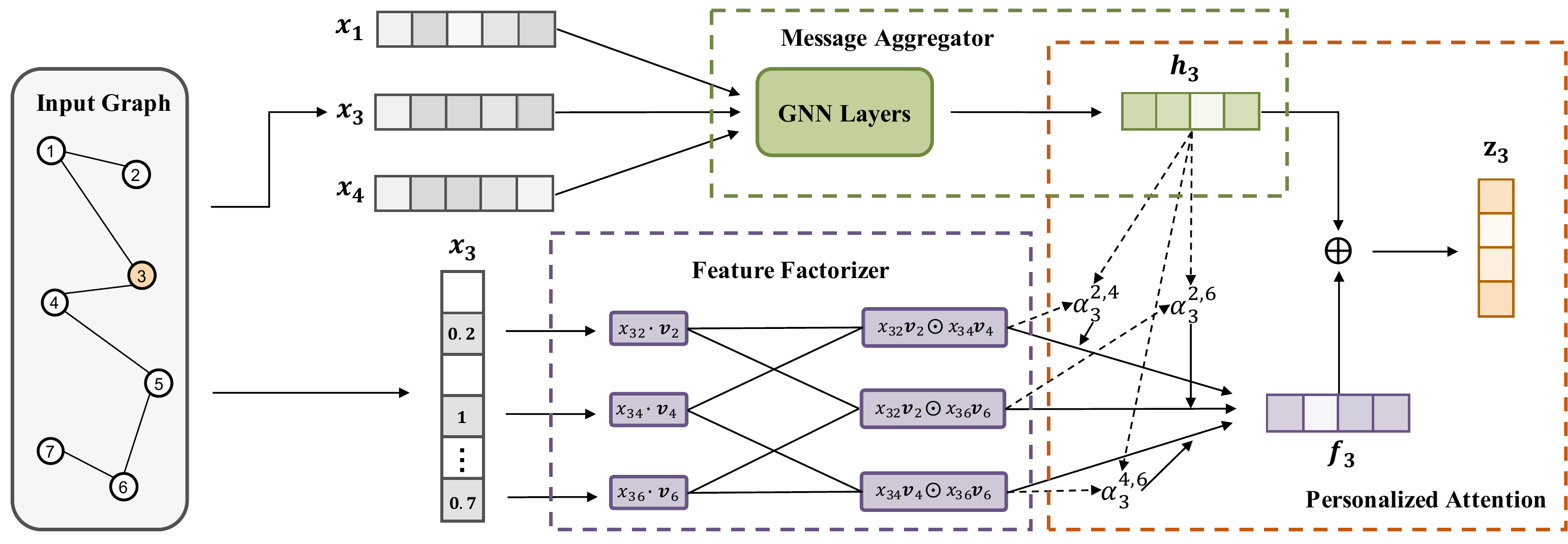}
    \caption{Illustration of the proposed feature interaction-aware graph neural networks (FI-GNNs).}%
    \label{fig:framework}%
\end{figure*}

In this section, we will illustrate the details of our feature interaction-aware graph neural networks (FI-GNNs). As shown in Figure \ref{fig:framework}, the FI-GNNs framework consists of three essential modules: (1) message aggregator; (2) feature factorizer; and (3) personalized attention. These components seamlessly model both raw features and feature interactions of nodes, yielding highly discriminative node representations.


\subsection{Message Aggregator} Our message aggregator is a GNN-based module that converts each node to a low-dimensional embedding via modeling the information from its raw features and the dependencies with its neighbors. Most of the prevailing GNN models follow the neighborhood aggregation strategy and are analogous to Weisfeiler-Lehman (WL) graph isomorphism test~\cite{xu2018powerful}. Specifically, the representation of a node is computed by iteratively aggregating representations of its local neighbors. In general, a GNN layer can be defined as:
\begin{equation}
\begin{aligned}
    \mathbf{h}_i^l &= \textsc{Combine}^l\Big( \mathbf{h}_i^{l-1},  \mathbf{h}_{\mathcal{N}_i}^l\Big),\\
    \mathbf{h}_{\mathcal{N}_i}^l &= \textsc{Aggregate}^l\Big(\{   \mathbf{h}_j^{l-1} \arrowvert \forall j \in \mathcal{N}_i \}\Big),
    \label{eq:graphSage}
\end{aligned}
\end{equation}
where $\mathbf{h}_i^{l}$ is the node representation of node $i$ at layer $l$ and $\mathcal{N}_i$ is the local neighbor set of node $i$. \textsc{Aggregate} and \textsc{Combine} are two key functions of GNNs and have a series of possible implementations~\cite{kipf2016semi,hamilton2017inductive,velivckovic2017graph}.

By stacking multiple GNN layers, the message aggregator captures the correlation between a node and its neighbors multiple hops away:
\begin{equation}
\begin{aligned}
    &\mathbf{H}^{1} = \text{GNN}^1(\mathbf{A}, \mathbf{X}),\\
    &\dots\\
    &\mathbf{H}^L = \text{GNN}^{L-1}(\mathbf{A}, \mathbf{H}^{L-1}).
\end{aligned}
\end{equation}

It should be mentioned that designing a different GNN layer is not the focus of this paper as we aim at empowering GNN models to learn node representation that incorporates feature interactions. In fact, any advanced aggregation function can be easily integrated into our framework, making the proposed FI-GNNs quite general and flexible to support learning on different types of graphs.

\subsection{Feature Factorizer}  
To automatically capture the interactions between node features, we propose to build a feature factorizer with two layers:



\smallskip
\noindent{\textbf{Sparse-feature Embedding Layer.}} In our feature factorizer, the first layer takes the sparse features of each node as input and projects each feature into a low-dimensional embedding vector through a neural layer. Formally, the $j$-th feature is projected to a $k$-dimensional dense vector $\mathbf{v}_j \in \mathbb{R}^{k}$. Thus for each node $i$, a set of embedding vectors $\mathcal{V}_i = \{x_{i1} \mathbf{v}_1, \dots, x_{id} \mathbf{v}_d \}$ is obtained to represent its features $\mathbf{x}_i$. It is worth noting that here we rescaled each embedding vector by its input feature value, which enables the feature factorizer to handle real-valued features~\cite{rendle2010factorization}. Also, due to the sparsity of the node features $\mathbf{x}_i$, we only need to consider the embedding vectors of those non-zero features for the sake of efficiency. 

\smallskip
\noindent{\textbf{Feature Factorization Layer.}} With the projected feature embedding vectors $\mathcal{V}_i$ of node $i$, the second layer aims to compute all feature interactions. Note that in this study we focus on the second-order (pair-wise) feature factorization, but the proposed model can be easily generalized to higher-order feature interactions. Following the idea of factorization machines~\cite{shan2016deep,blondel2016polynomial,he2017neural} which adopt the inner product to model the interaction between each pair of features, we propose to represent the pair-wise interaction between feature $x_{ij_1}$ and $x_{ij_2}$ as $x_{ij_1} \mathbf{v}_{j_1}  \odot  x_{ij_2} \mathbf{v}_{j_2}$, where $\odot$ denotes the element-wise product of two vectors. The output of our feature factorizer can be defined as:

\begin{equation}\label{eq:1}
    \mathcal{F}_i  = \{x_{ij_1} \mathbf{v}_{j_1}  \odot  x_{ij_2} \mathbf{v}_{j_2} | \forall (j_1, j_2) \in \mathcal{R}_x \},
\end{equation}
where $\mathcal{R}_x = \{(i, j)\}_{i \in d , j \in d, j>i}$ for simplicity. 

\subsection{Personalized Attention}
However, we observe that different feature interactions usually exhibit different levels of informativeness for characterizing a node. In a social network, consider there is a female user Alice, who is 35 years old and has high income, the feature interaction \textit{gender\_income} = \textit{F\_high} is more informative for predicting her interest as ``Jewelry", while the feature interaction \textit{gender\_age} = \textit{F\_35} is less informative since this interest is not very related to age for a woman; Reversely, for another male user Bob, who is 20 years old with low income, the feature interaction \textit{gender\_age} = \textit{M\_20} is more informative for predicting his interest as ``Sports", while the feature interaction \textit{gender\_income} = \textit{M\_low} has less impact since this interest is not quite related to income for a man.


Thus, for different nodes, the informativeness of a feature interaction could vary a lot. In this module, we propose to use a personalized attention mechanism to recognize and highlight informative feature interactions for each node. Specifically, the attention weight of each factorized feature embedding is calculated based on its interaction with the aggregated node features $\mathbf{h}_i$. The attention weight of feature interaction $x_{ij_1} \mathbf{v}_{j_1}  \odot  x_{ij_2} \mathbf{v}_{j_2}$ can be computed by:
\begin{equation}
\begin{aligned}
    a_i^{j_1, j_2} &= \mathbf{h}_i^{\mathrm{T}} \tanh(\mathbf{W}_f  (x_{ij_1} \mathbf{v}_{j_1}  \odot  x_{ij_2} \mathbf{v}_{j_2})),\\
    \alpha_i^{j_1, j_2} &= \frac{\exp(a_i^{j_1, j_2})}{\sum_{(k_1,k_2) \in \mathcal{R}_x} \exp(a_i^{k_1,k_2})},
\end{aligned}
\end{equation}
where $\mathbf{W}_f \in \mathbb{R}^{k \times k}$ is a projection matrix.
The final representation of feature interaction  $\mathbf{f}_i$ can be computed by:
\begin{equation}
    \mathbf{f}_i = \sum_{j_1 = 1}^{d} \sum_{j_2=j_1+1}^{d}  \alpha_i^{j_1,j_2}  (x_{i j_1} \mathbf{v}_{j_1}  \odot  x_{ij_2} \mathbf{v}_{j_2}).
\end{equation}

In this way, the personalized attention module will put fine-grained weights on each factorized feature interaction for each node. Finally, we concatenate $\mathbf{h}_i$ and $\mathbf{f}_i$ as the feature interaction-aware node representation $\mathbf{z}_i = \mathbf{h}_i \oplus \mathbf{f}_i$.

\subsection{Model Learning}
Based on the feature interaction-aware node representations, we are able to design different task-specific loss functions to train the proposed model. It is worth mentioning that FI-GNNs are a family of models which could be trained in a supervised,
semi-supervised, or unsupervised setting. For instance, the cross-entropy over all labeled data is employed as the loss function of the semi-supervised node classification problem~\cite{kipf2016semi}:
\begin{equation}
    \mathcal{L}_{semi} = - \frac{1}{C} \sum_{l \in \mathcal{Y}_L} \sum_{c = 1}^{C} \mathbf{Y}_{lc} \text{log}( \widehat{\mathbf{Y}}_{lc}),
\end{equation}
where $C$ is the number of class labels and $\mathcal{Y}_L$ is the set of annotated node indices in the input graph and $\widehat{\mathbf{Y}}_{lc} = \text{softmax}(\mathbf{z}_l)$.
By minimizing the loss function, we are able to predict labels of those nodes that are not included in $\mathcal{Y}_L$.

Moreover, if we aim at learning useful and predictive representations in a fully unsupervised setting~\cite{hamilton2017inductive}, the loss function can be defined as follows:
\begin{equation}
    \mathcal{L}_{un} = - \sum_{(i, j) \in \mathcal{E}} \log(\sigma(\mathbf{z}_i^{\mathrm{T}} \mathbf{z}_j)) -
    \sum_{(i, k) \in \mathcal{E}^{-}}
    \log(\sigma(-\mathbf{z}_i^{\mathrm{T}} \mathbf{z}_k)),
\end{equation}
where $\mathcal{E}^{-}$ is the negatively sampled edges that do not appear in the input graph
and here $\sigma$ is the sigmoid function. Briefly, the unsupervised loss function $\mathcal{L}_{un}$ encourages linked nodes to have similar representations, while enforcing that the representations of disparate nodes are highly distinct. 
This unsupervised setting emulates situations where node representations are provided to downstream
machine learning applications, such as link prediction and node clustering.


\subsection{Theoretical Analysis}
So far we have illustrated the details of FI-GNNs for learning discriminative node representations that incorporate feature interactions. Next, we show the connection between our proposed framework and the vanilla factorization machine. 


\begin{lemma}
FI-GNNs are the generalization of factorization machines on graph-structured data. An FI-GNN model is equivalent to the vanilla factorization machine by ignoring node dependencies.
\end{lemma}
\begin{proof}
Here we take the FI-GCN with a one-layer message aggregator as an example. First, we ignore the attention weights in the personalized attention and directly project the concatenated embeddings to a prediction score, this simplified model can be expressed as:
\begin{equation}
     \hat{y}_i = w_0 + \mathbf{u}^{\mathrm{T}} \bigg(\mathbf{h}_i \oplus \sum_{j_1 = 1}^{d} \sum_{j_2=j_1+1}^{d}  x_{ij_1} \mathbf{v}_{j_1} \odot  x_{ij_2} \mathbf{v}_{j_2} \bigg),
     \label{eq14}
\end{equation}
where $w_0$ is the global bias and $\mathbf{u} \in \mathbb{R}^{2k}$ is the projection vector. By neglecting the dependencies between nodes, we can directly get:
\begin{equation}
     \hat{y}_i = w_0 + \mathbf{u}^{\mathrm{T}} \bigg(\mathbf{W}  \mathbf{x}_i \oplus \sum_{j_1 = 1}^{d} \sum_{j_2=j_1+1}^{d}  x_{ij_1} \mathbf{v}_{j_1} \odot  x_{ij_2} \mathbf{v}_{j_2} \bigg).
     \label{eq15}
\end{equation}


If we further fix $\mathbf{u}$ to a constant vector of $[1, \dots, 1] \in \mathbb{R}^{2k}$, Eq. (\ref{eq15}) can be reformulated as:
\begin{equation}
\begin{aligned}
\hspace{-0.08cm}     \hat{y}_i &= w_0 + \mathbf{1}^{\mathrm{T}} \mathbf{W} \mathbf{x}_i +   \mathbf{1}^{\mathrm{T}} \sum_{j_1 = 1}^{d} \sum_{j_2=j_1+1}^{d}  x_{ij_1} \mathbf{v}_{j_1} \odot  x_{ij_1} \mathbf{v}_{j_2}\\
     &= w_0 + \mathbf{w}^{\mathrm{T}} \mathbf{x}_i + 
      \sum_{j_1 = 1}^{d} \sum_{j_2 = j_1+1}^{d} \langle \mathbf{v}_{j_1} \mathbf{v}_{j_2} \rangle \cdot x_{ij_1} x_{ij_2},
\end{aligned}
\end{equation}
where $\mathbf{1} \in \mathbb{R}^k$ and $\mathbf{w}^{\mathrm{T}} = \mathbf{1}^{\mathrm{T}} \mathbf{W} $. Thus the vanilla FM model can be exactly recovered.
\end{proof}

\section{Experiments}

In this section, we conduct extensive experiments to evaluate the effectiveness of the proposed framework FI-GNNs.

\subsection{Experiment Settings}
\smallskip
\noindent \textbf{Evaluation Datasets.} To comprehensively understand how FI-GNNs work, we adopt four public real-world attributed network datasets~(BlogCatalog~\cite{li2015unsupervised}, Flickr~\cite{li2015unsupervised}, ACM~\cite{tang2008arnetminer} and DBLP~\cite{kipf2016semi}). 
Table \ref{table:datasets} reports brief statistics for each dataset. Note that the node features of all the above datasets are generated by the bag-of-words model, yielding high-dimensional and sparse node features. As listed in Table \ref{table:datasets}, the average number of non-zero features (nnz) of nodes is significantly smaller than the number of feature dimensions. 

\smallskip
\noindent \textbf{Compared Methods.} For performance comparison, we compare the proposed framework FI-GNNs with three types of baselines: (1) \textit{Random walk-based methods} including DeepWalk~\cite{perozzi2014deepwalk} and node2vec~\cite{grover2016node2vec}; (2) \textit{Factorization machine-based methods} including the vanilla Factorization Machine (FM)~\cite{rendle2010factorization} and neural factorization machine~(NFM)~\cite{he2017neural}; and (3) \textit{GNN-based methods} including GCN~\cite{kipf2016semi} and GraphSAGE~\cite{hamilton2017inductive} with mean pooling. In order to show the effectiveness and flexibility of the proposed framework, we compare above baseline methods with two different instantiations of FI-GNNs, i.e., FI-GCN and FI-GraphSAGE.

\begin{table}[t!]
\caption{Statistics of the four real-world datasets.}
\centering
\scalebox{0.86}{
\begin{tabular}{@{}lccccc@{}}
\toprule

\multirow{2}{*}{} & \multicolumn{2}{c}{Social Networks} & & \multicolumn{2}{c}{Citation Networks} \\
\cline{2-3} \cline{5-6}

\rule{0pt}{10pt} & BlogCatalog & Flickr & & ACM &  DBLP \\ \midrule
\# nodes & 5,196 & 7,575 & & 16,484 & 18,448   \\
\# edges & 173,468 & 242,146 & &  71,980  & 44,338   \\
\# attributes     & 8,189  & 12,047  & &  8,337  & 2,476\\
\# degree (avg)  & 66.8  & 63.9  &  &  8.7 & 4.8\\
\# nnz (avg)  & 71.1  & 24.1  &  &  28.4 & 5.6\\
\# labels & 6 & 9 & & 9 & 4   \\
\bottomrule
\end{tabular}
}

\label{table:datasets}
\end{table}


        
    
    
    
    

\begin{table*}[t!]
\centering
\caption{Semi-supervised node classification results on four datasets w.r.t ACC and F1 (\%).}
\scalebox{0.86}{
\begin{tabular}{@{}lcccccccccccccccc@{}}
\toprule
\rule{0pt}{10pt} \multirow{2}{*}{Methods}  & \multicolumn{2}{c}{BlogCatalog}  & &  \multicolumn{2}{c}{Flickr}  & & \multicolumn{2}{c}{ACM}  & & \multicolumn{2}{c}{DBLP} \\ \cline{2-3} \cline{5-6} \cline{8-9} \cline{11-12}

\rule{0pt}{10pt} & \multicolumn{1}{c}{ACC}  & \multicolumn{1}{c}{F1} & & \multicolumn{1}{c}{ACC} & \multicolumn{1}{c}{F1} & &
\multicolumn{1}{c}{ACC} & \multicolumn{1}{c}{F1}  & &
\multicolumn{1}{c}{ACC} & \multicolumn{1}{c}{F1}

\\ \hline

DeepWalk        & 59.5  & 60.3 & & 46.1 & 45.4 & & 57.2 & 54.7 & & 65.7 & 61.9    \\
node2vec        & 61.7 & 61.6 & & 43.3 & 40.3 & & 56.9 & 54.0 && 64.4 & 63.8 \\
\hline

FM  & 60.6 &  60.2 && 43.0 & 42.8 && 42.9 & 36.6 && 65.4 & 57.9 \\
NFM  & 61.3 & 61.2 && 44.2 & 43.6 && 44.8 & 43.3 && 67.1 & 66.8         \\
\hline

GCN        &  73.5 & 73.0 && 52.7 & 51.1 && 72.9 & 71.0 && 83.5 & 83.3   \\
GraphSAGE   &  78.8  & 78.3 && 71.7 & 71.4 && 59.3 & 58.8 && 72.6 & 71.8   \\
\hline


FI-GCN  &  80.1({\small +9.0\%}) & 80.0({\small +9.6\%}) &&  55.8({\small +5.9\%}) & 55.4({\small +8.4\%}) && 74.2({\small +1.8\%}) & 73.6({\small+3.7\%}) && 84.6({\small +1.3\%}) & 84.4({\small +1.3\%})\\

FI-GraphSAGE   & 86.3({\small +9.5\%}) & 86.1({\small +9.9\%}) && 73.8({\small +2.9\%}) & 73.8({\small +3.7\%}) && 61.7({\small +4.0\%}) & 59.8({\small +1.7\%}) && 74.1({\small +2.1\%}) & 73.0({\small +1.7\%})   \\

\bottomrule
\end{tabular}}

\label{table:semi}

\end{table*}




\smallskip
\noindent{\textbf{Implementation Details.}} 
The two instantiations of FI-GNNs are implemented in PyTorch~\cite{paszke2017automatic} with Adam optimizer~\cite{kingma2014adam}. For each of them, we build the message aggregator with two corresponding GNN layers (32-neuron and 16-neuron, respectively) and each layer has the ReLU activation function. In addition, the embedding size in our feature factorizer is set to 16. 
For the baseline methods, we use their released implementations, and the embedding size of these algorithms are set to 32 for a fair comparison. FI-GNNs are trained with a maximum of 200 epochs and an early stopping strategy. The learning rate and dropout rate of FI-GNNs are set to 0.005 and 0.1, respectively. The hyper-parameters (e.g., learning rate, dropout rate) of baseline methods are selected according to the best performance on the validation set. 



\begin{table*}[t!]
\centering

\caption{Unsupervised link prediction results on four datasets w.r.t AUC and AP (\%).}
\scalebox{0.86}{
\begin{tabular}{@{}lcccccccccccccccc@{}}
\toprule
\rule{0pt}{10pt} \multirow{2}{*}{Methods}  & \multicolumn{2}{c}{BlogCatalog}  & &  \multicolumn{2}{c}{Flickr}  & & \multicolumn{2}{c}{ACM}  & & \multicolumn{2}{c}{DBLP} \\ \cline{2-3} \cline{5-6} \cline{8-9} \cline{11-12}

\rule{0pt}{10pt} & \multicolumn{1}{c}{AUC}  & \multicolumn{1}{c}{AP} & & \multicolumn{1}{c}{AUC} & \multicolumn{1}{c}{AP} & &
\multicolumn{1}{c}{AUC} & \multicolumn{1}{c}{AP}  & &
\multicolumn{1}{c}{AUC} & \multicolumn{1}{c}{AP}

\\ \hline
DeepWalk        & 65.0 & 63.8 & & 56.9 & 52.4 && 79.1 & 75.3 && 66.1 & 61.2   \\

node2vec & 60.9  & 61.5 & & 51.9 & 56.5 && 69.2 & 69.7 && 79.8 & 78.2 \\
\hline

FM & 62.7 & 61.2 && 80.4 & 79.6 && 50.3 &  50.1 && 52.4 & 52.3\\
NFM  & 64.8 & 64.0 && 84.5 & 83.9 && 50.5 &  50.4 && 54.8 & 54.4 \\
\hline

GCN        &  81.1  & 81.1 && 90.1 & 91.1 && 92.4 & 91.1 && 85.6 & 84.4   \\
GraphSAGE   &  71.7  & 70.3 && 86.1 & 86.8 && 85.9 & 83.9 && 85.2 & 84.0   \\
\hline


FI-GCN  &  85.9({\small +5.9\%})  & 85.9({\small +5.9\%}) && 93.0({\small +3.2\%}) & 92.5({\small +1.5\%}) && 93.3({\small +1.0\%}) & 92.9({\small +2.0\%}) && 88.8({\small +3.7\%}) & 89.5({\small +6.0\%}) \\
FI-GraphSAGE   & 76.0({\small+6.0\%})  & 75.0({\small +6.7\%}) && 89.0({\small +3.4\%}) & 88.7({\small +2.2\%}) && 87.7({\small +2.1\%})  & 86.8({\small +3.5\%}) && 88.7({\small +4.1\%}) & 86.0({\small +2.4\%})   \\

\bottomrule
\end{tabular}}

\label{table:unsup}

\end{table*}

\subsection{Evaluation Results}
\smallskip
\noindent{\textbf{Semi-Supervised Node Classification.}} The first evaluation task is semi-supervised node classification, which aims to predict the missing node labels with a small portion of labeled nodes. For each dataset, given the entire nodes $\mathcal{V}$, we randomly sample 10\% of $\mathcal{V}$ as the training set and use another 20\% nodes as the validation set for hyper-parameter optimization. During the training process, the algorithm has access to all of the nodes' feature vectors and the network structure. The predictive power of the trained model is evaluated on the left 70\% nodes.
We repeat the evaluation process 10 times and report the average performance in terms of two evaluation metrics (Accuracy and Micro-F1) in Table~\ref{table:semi}. Note that we also include the performance improvements of FI-GNNs over their corresponding GNN models in this table. The following findings can be inferred from the table:
\begin{itemize}

    

    
    \item The two instantiations of FI-GNNs (FI-GCN and FI-GraphSAGE) outperform their corresponding GNN models (GCN and GraphSAGE) on all the four datasets. In particular, FI-GNNs achieve above 9\% performance improvement over GNN models on the BlogCatalog dataset. It validates the necessity of incorporating informative feature interactions into node representation learning on the feature-sparse graphs.
    

    
    \item Overall, FI-GNNs achieve higher improvements on social networks than citation networks in our experiments. According to the finding from previous research~\cite{li2017attributed}, the class labels in social networks are more closely related to the node features, while the class labels in citation networks are more closely related to the network structure. It explains why our FI-GNN models are more effective in social network data by considering informative feature interactions.

\end{itemize}


\smallskip
\noindent{\textbf{Unsupervised Link Prediction.}} The second evaluation task is unsupervised link prediction, which tries to infer the links that may appear in a foreseeable future. Specifically, we randomly sample 80\% edges from $\mathcal{E}$ and an
equal number of nonexistent links as the training set. Meanwhile, another two sets of 10\% existing links and an equal number of nonexistent links are used as validation and test sets. We conduct each experiment 10 runs and report the test set performance when the best performance on the validation set is achieved. The average performance is reported with the area under curve score (AUC) and average precision (AP). The experimental results on link prediction are shown in Table \ref{table:unsup}. The performance improvements of FI-GNNs over corresponding GNN models are included. Accordingly, we make the following observations:

\begin{itemize}

    \item The FI-GNN models are able to achieve superior link prediction results over other baselines. For instance, FI-GCN and FI-GraphSAGE improve around 6\% over GCN and GraphSAGE on the BlogCatalog dataset. The experimental results successfully demonstrate that the proposed FI-GNNs are able to learn more discriminative node representations on feature-sparse graphs under the unsupervised setting than existing GNN models.

    \item Both the \textit{random walk-based methods} (DeepWalk and node2vec) and \textit{factorization machine-based methods} show unsatisfactory performance in our experiments. The main reason is that \textit{random walk-based methods} are unable to consider node attributes and \textit{factorization machine-based methods} are unable to consider node dependencies for learning node repesentations.

    


\end{itemize}
\begin{figure}[!b]
    \graphicspath{{figures/}}
    \centering
    \vspace{-0.4cm}
    \subfigure[node classification]
    {
     \includegraphics[width=0.5\columnwidth]{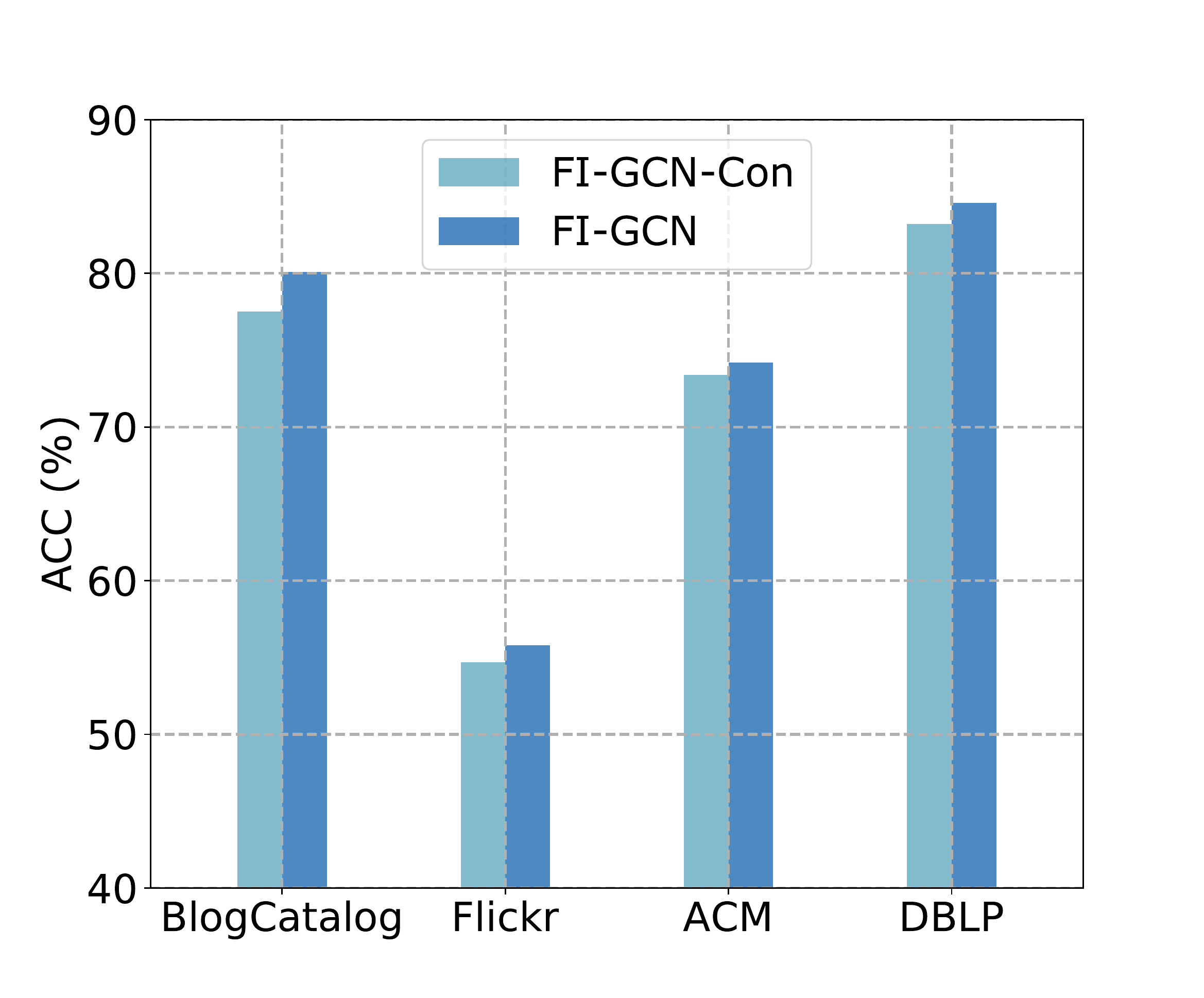}
    }
    \hspace{-0.52cm}
    \subfigure[link prediction]
    {
     \includegraphics[width=0.5\columnwidth]{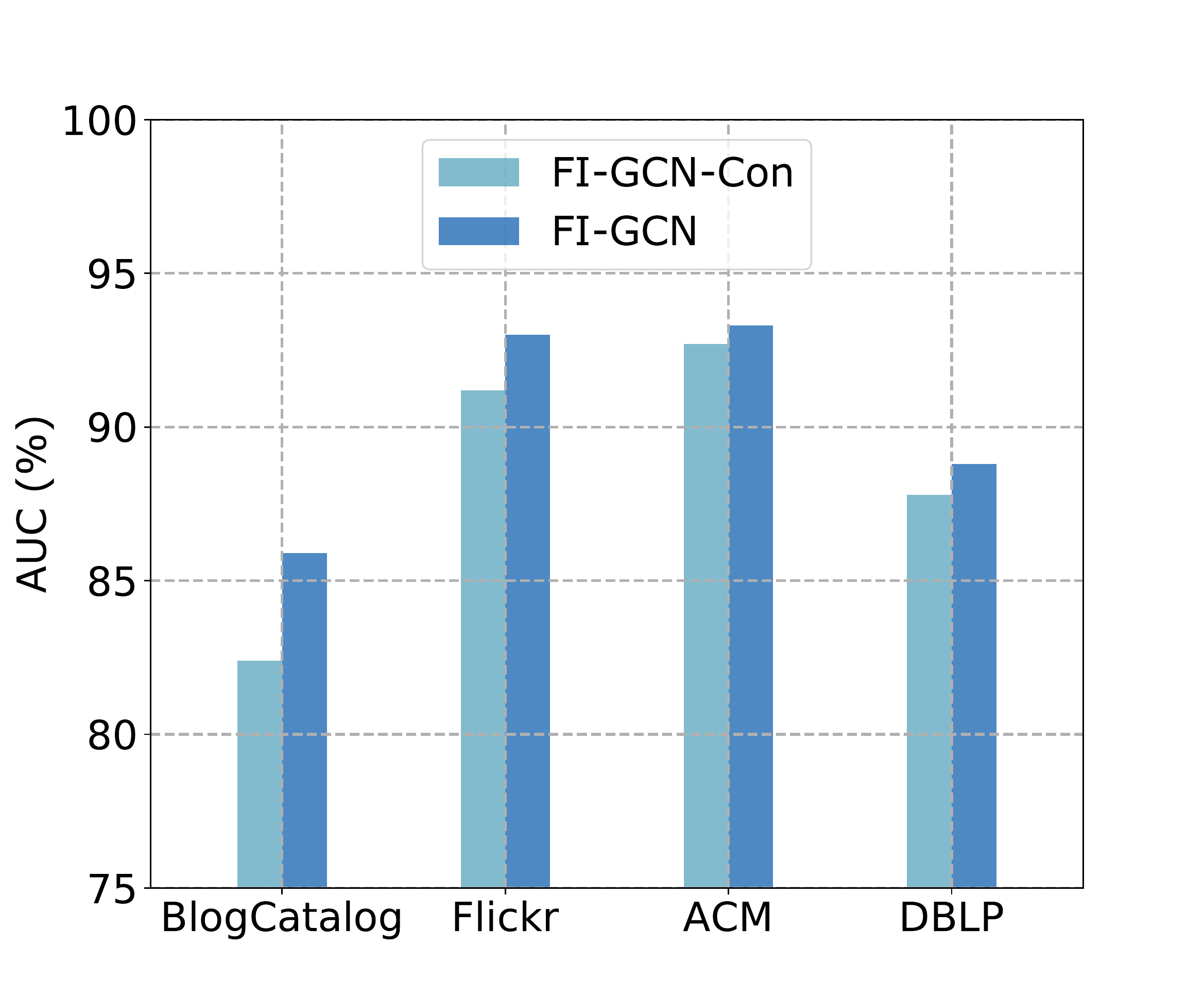}
    }
    \vspace{-0.3cm}
    \caption{The comparison results of ablation study.}%
    \label{fig:ablation}%
\end{figure}

\begin{figure*}[!t]
    
    \graphicspath{{figures/}}
    \centering
    \subfigure[DeepWalk] 
    {
    \includegraphics[width=0.475\columnwidth]{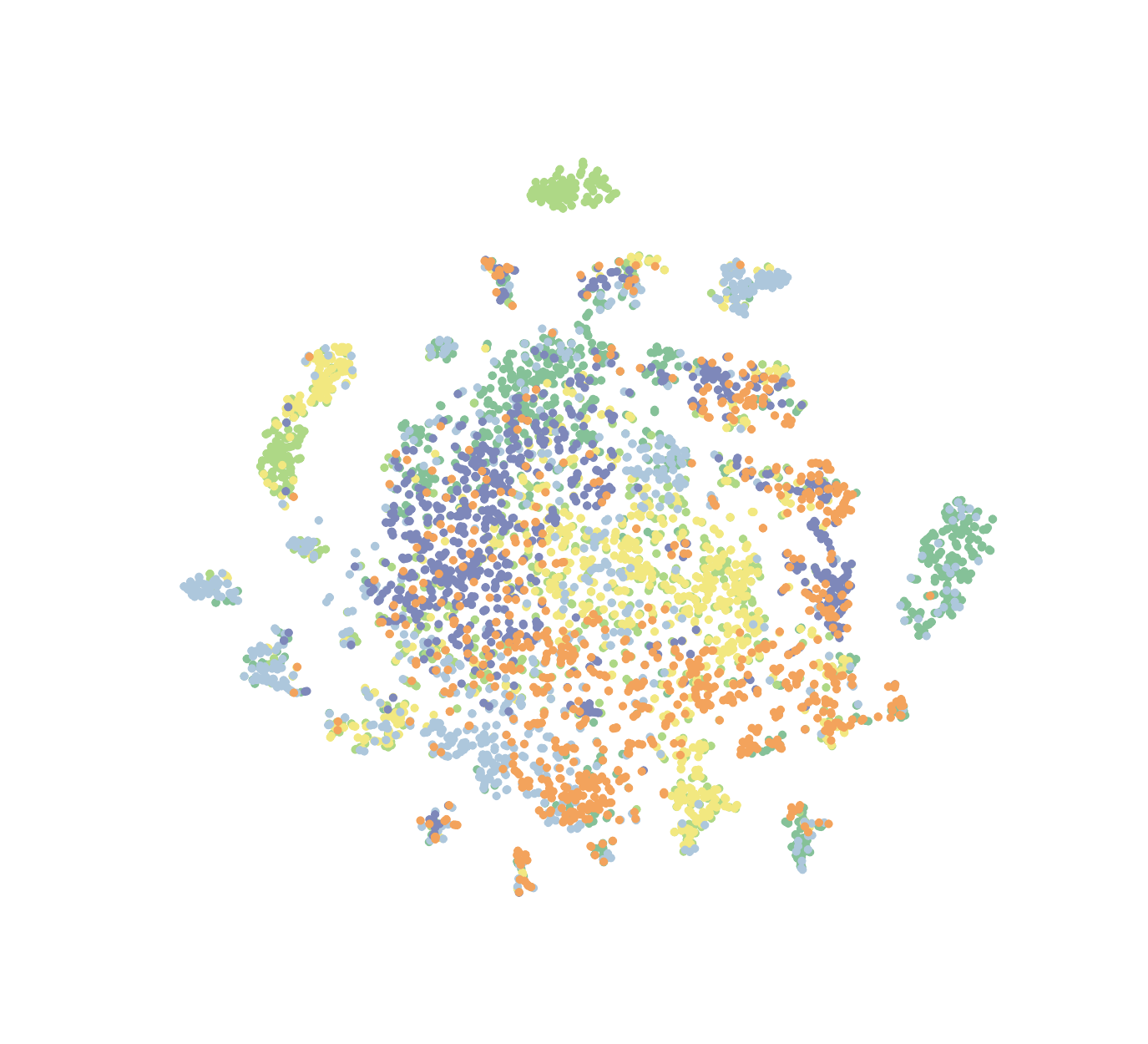}
    }
    \hspace{-0.5cm}
    \subfigure[node2vec]
    {
    \includegraphics[width=0.475\columnwidth]{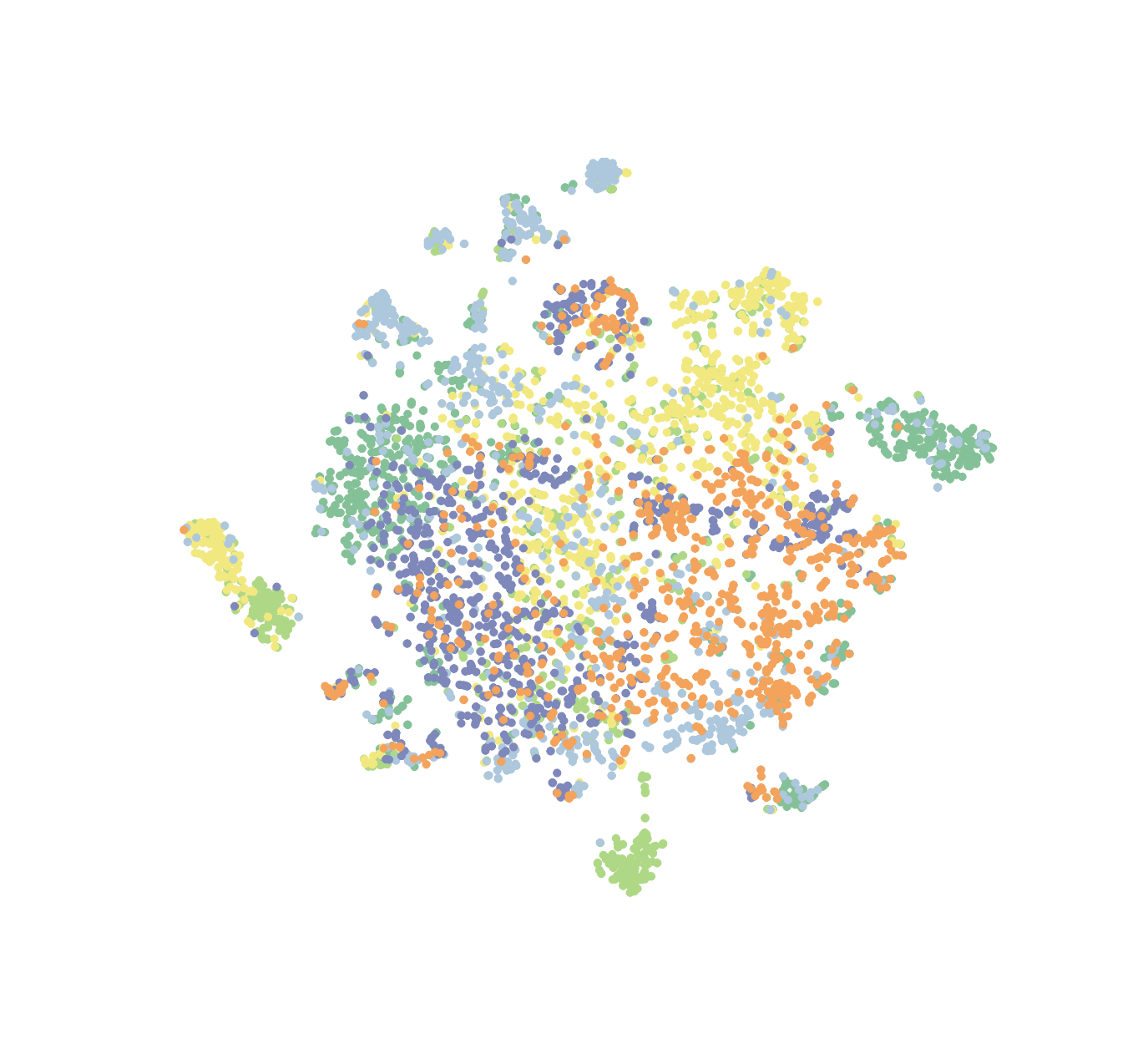}
    }
    \hspace{-0.5cm}
    \subfigure[GCN] 
    {
    \includegraphics[width=0.475\columnwidth]{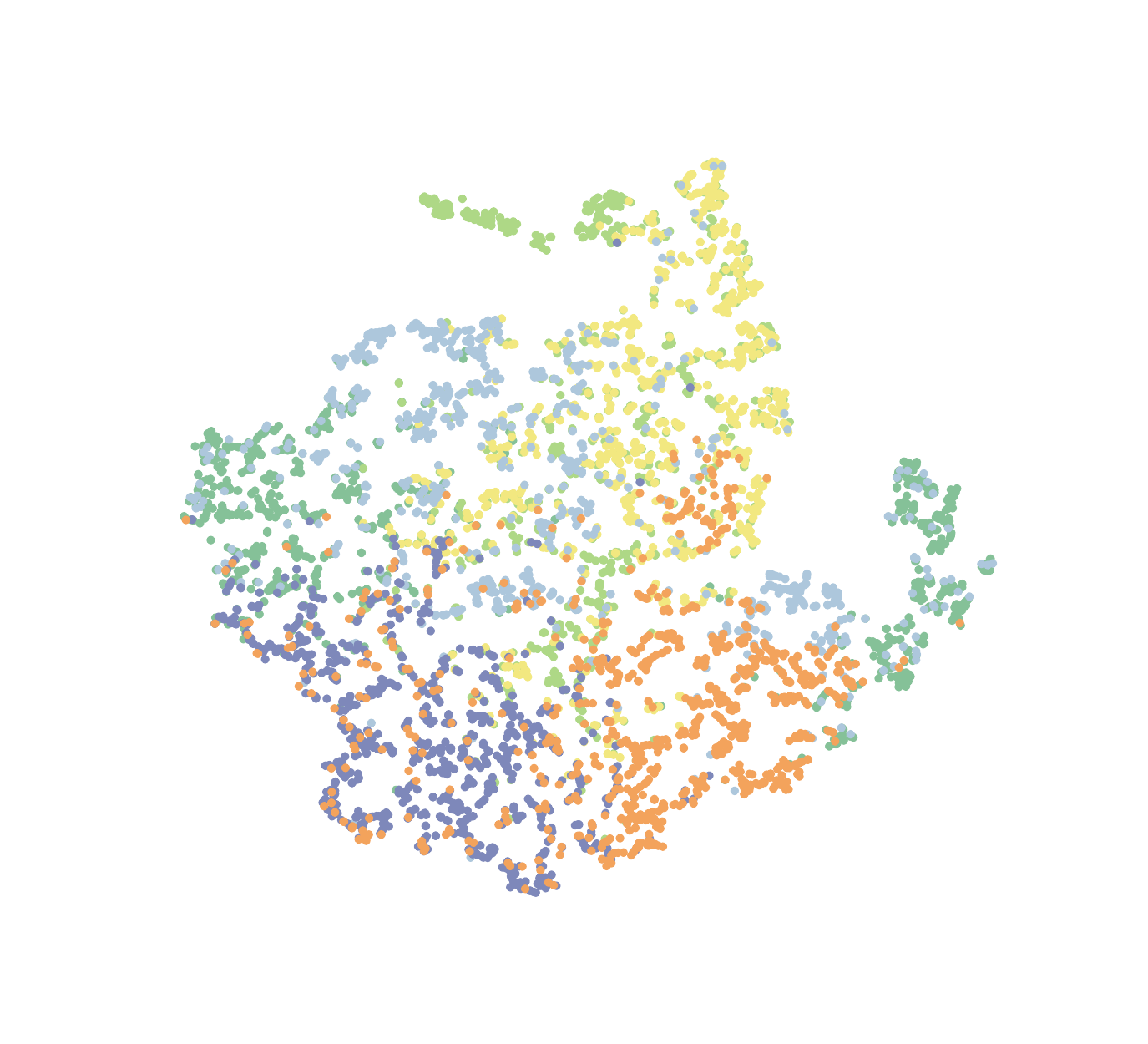}
    }
    \hspace{-0.5cm}
    \subfigure[FI-GCN]
    {
    \includegraphics[width=0.475\columnwidth]{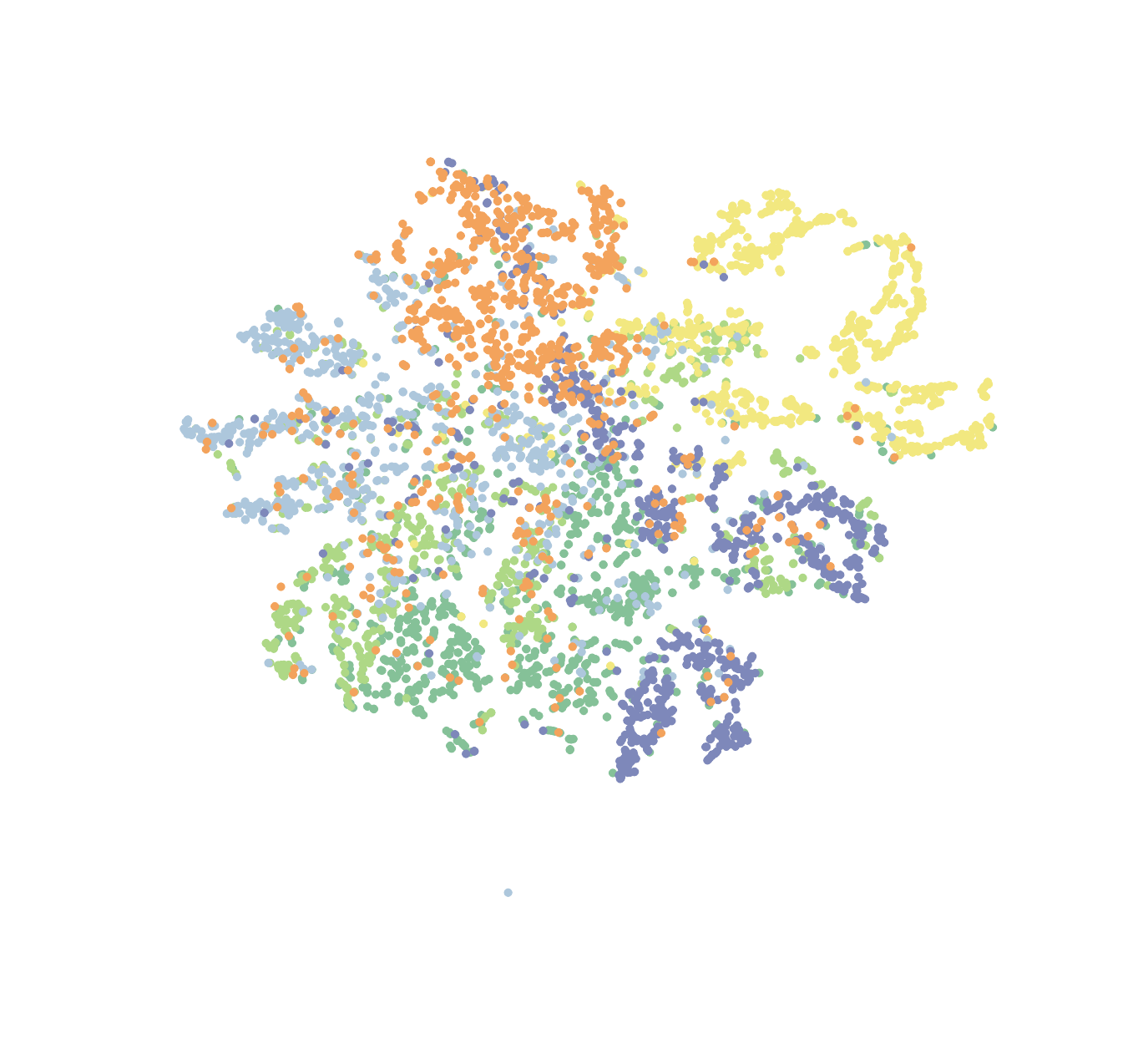}
    }
    \caption{Embedding visualization of different methods for the BlogCatalog dataset.}%

    \label{fig:embed_vis}
\end{figure*} 

\subsection{Further Analysis}

\smallskip
\noindent{\textbf{Ablation Study.}} 
From the previous experiment results, we already show the effectiveness of the message aggregator and the feature factorizer. Here we conduct an ablation study to investigate the effect of the personalized attention. Specifically, we compare FI-GCN with a variant FI-GCN-Con, which excludes the personalized attention by directly summing the feature interaction representation. The comparison results of two evaluation tasks (semi-supervised node classification and unsupervised link prediction) on four datasets are presented in Figure \ref{fig:ablation}. From the reported results, we can clearly observe that FI-GCN-Con falls behind FI-GCN with a noticeable margin on the two evaluation tasks, which demonstrates the effectiveness of our proposed personalized attention.

\smallskip
\noindent{\textbf{Embedding Visuliazation.}} To further show the embedding quality of the proposed framework FI-GNNs, we use t-SNE~\cite{maaten2008visualizing} to visualize the extracted node representations from different models. Due to the
space limitation, we only post the results of DeepWalk, node2vec, GCN, and FI-GCN for the BlogCatalog dataset under the semi-supervised setting. The visualization results are shown in Figure \ref{fig:embed_vis}. Note that the node's color
corresponds to  its label, verifying the model's discriminative power across the six user groups of BlogCatalog. As observed in Figure \ref{fig:embed_vis}, the random walk-based methods (e.g., DeepWalk, node2vec) cannot effectively identify different classes. Despite GCN improves the embedding quality by incorporating the node features in the learning process, the boundary between different classes is still unclear. FI-GCN performs the best as it can achieve more compact and separated clusters.


    
    

\section{Related Work}

\smallskip
\noindent{\textbf{Graph Neural Networks.}}
Graph neural networks (GNNs), a family of neural models for learning latent node representations in a graph, have achieved remarkable success in different graph learning tasks~\cite{scarselli2008graph,bruna2013spectral,defferrard2016convolutional,kipf2016semi,velivckovic2017graph}. Most of the prevailing GNN models follow the neighborhood aggregation scheme, learn the latent node representation via message passing among local neighbors in the graph. Though based on graph spectral theory, the learning process of graph convolutional networks (GCNs)~\cite{kipf2016semi} also can be considered as a mean-pooling neighborhood aggregation; GraphSAGE~\cite{hamilton2017inductive} concatenates the node's feature in addition to mean/max/LSTM pooled neighborhood information; Graph Attention Networks (GATs) incorporate trainable attention weights to specify fine-grained weights on neighbors when aggregating neighborhood information of a node. Recent research further extend GNN models to consider global graph information~\cite{battaglia2018relational} and edge information~\cite{gilmer2017neural} during aggregation. However, the aforementioned GNN models are unable to capture 
high-order signals among features. Our approach tackles this problem by incorporating feature interactions into graph neural networks, which enables us to learn more expressive node representations for different graph mining tasks.

\smallskip
\noindent{\textbf{Factorization Machines.}}
The term Factorization Machine (FM) is first proposed in \cite{rendle2010factorization}, combining the key ideas of factorization models (e.g., MF, SVD) with general-purpose machine learning techniques~\cite{scholkopf2001learning}. In contrast to conventional factorization models, FM is a general-purpose predictive framework for arbitrary machine learning tasks and characterized by its
usage of the inner product of factorized parameters to model
pairwise feature interactions. More recently, FMs have received neural makeovers due to the success of deep learning. Specifically, Neural Factorization Machines (NFMs)~\cite{he2017neural} was proposed to enhance the expressive ability of standard FMs with
nonlinear hidden layers for sparse predictive analytics. DeepFM~\cite{guo2017deepfm} is another model that combines the prediction scores of a
deep neural network and FM model for CTR prediction. Moreover, Xiao et al.~\cite{xiao2017attentional} equipped FMs with
a neural attention network to discriminate the importance of each feature interaction, which not only improves the representation ability but also the interpretability of an FM model. Despite the effectiveness of modeling various feature interactions, existing FM variants are unable to handle graph-structured data due to the inability of modeling complex dependencies amongst nodes. Our proposed framework can be considered as a powerful extension of FMs on graph-structured data.


\section {Conclusion}

In this paper, we propose Feature Interaction-aware Graph Neural Networks (FI-GNNs), a novel graph neural network framework for computing node representations that incorporates informative feature interactions. Specifically, the proposed framework extracts a set of feature interactions for each node and highlights those informative ones based on personalized attention. This plug-and-play framework that can be extended to various types of graphs for different predictive tasks, which is highly easy-to-use and flexible. Experimental results on four real-world datasets demonstrate the effectiveness of our proposed framework. 

\newpage
\balance
\bibliographystyle{named}
\bibliography{ijcai20}

\end{document}